\documentclass[preprint,12pt,authoryear]{elsarticle}

\usepackage{subfig}
\usepackage{amssymb}
\usepackage[tbtags]{amsmath}
\usepackage{booktabs}

\usepackage[colorlinks,linkcolor=blue, anchorcolor=red,citecolor=blue,CJKbookmarks=true]{hyperref}
\newtheorem{theorem}{Theorem}[section]

\usepackage{algorithm}
\usepackage{algpseudocode}

\journal{Journal of Computational and Applied Mathematics}

\begin{document}
	
\begin{frontmatter}
	\title{ An Interpretable and Stable Framework for Sparse Principal Component Analysis}
	\author[a]{Ying Hu}

	\author[a]{Hu Yang\corref{cor1}}
	\cortext[cor1]{Corresponding author}
	\ead{yh@cqu.edu.cn}
	\address[a]{College of Mathematics and Statistics, Chongqing University, Chongqing, 401331, China}
\begin{abstract}
Sparse principal component analysis (SPCA) solves the problem of poor interpretability and variable redundancy that principal component analysis (PCA) often faces with high-dimensional data. However, SPCA imposes uniform penalties on variables and lacks differential modeling of variable importance, leading to unstable performance in highly noisy or structurally complex data. We propose SP-SPCA  that introduces a single equilibrium parameter, which incorporates a single balancing parameter into the regularization framework to adjust variable penalties adaptively. This modification of the$L_2$ penalty enables flexible control over the trade-off between sparsity and variance explained while maintaining computational efficiency. Simulation studies show that our method consistently outperforms standard sparse principal component techniques in identifying sparse loading patterns, filtering noise variables, and preserving cumulative variance, especially under high-dimensional and noisy conditions. Empirical applications to crime and financial market datasets confirm the method’s practical utility. In applications to real datasets, the method selects fewer but more relevant variables, reducing complexity while maintaining explanatory power. Overall, our approach provides a robust and efficient alternative for sparse modeling in complex, high-dimensional data, offering clear advantages in stability, feature selection, and interpretability.
\end{abstract}

\end{frontmatter}

\section{Introduction}\label{sec1}

Principal Component Analysis (PCA) has long been recognized as a fundamental and effective technique for dimensionality reduction in data science and multivariate statistical analysis. PCA constructs a set of orthogonal linear combinations of the original variables—known as principal components—that aim to capture as much of the data variance as possible(\cite{Pearson1901}). However, traditional PCA expresses each component as a linear combination of all variables, regardless of their individual contributions. As a result, many variables with negligible loadings remain involved, making the interpretation of principal components difficult, particularly in high-dimensional settings. This issue was highlighted early on in the work of \cite{Jeffers1967}.

To improve the interpretability of PCA, various sparsity-inducing modifications have been proposed. \cite{Cadima1995} introduced an early idea of sparse PCA by thresholding small loadings to zero. However, this approach lacked a principled rule for selecting the threshold, often leading to instability and the potential exclusion of informative variables. Later, \cite{Jolliffe2003} proposed the SCoTLASS method, which incorporated an $L_1$-norm penalty directly into the PCA objective function. This allowed sparse solutions to emerge naturally through optimization. Yet, the resulting problem was non-convex, making solutions sensitive to initialization and lacking clear guidance on parameter tuning.

In 2006, \cite{Z2006} proposed Sparse Principal Component Analysis (SPCA) by reformulating PCA as a regression problem with an elastic net penalty. SPCA successfully combined variable selection and principal component extraction, and has since spurred a series of developments. To address non-convexity, \cite{DAspremont2007} introduced DSPCA using semidefinite relaxation. \cite{Shen2008} proposed a sparse PCA framework (sPCA-rSVD) by reformulating PCA as a regularized low-rank matrix approximation problem, introducing sparsity through thresholded singular vectors with efficient iterative updates.\cite{Qi2013} later proposed a hybrid-norm SPCA model that combined $L_1$ and $L_2$ penalties, offering more flexible and accurate control over sparsity. More recently, \cite{MC2019} developed PSPCA, which transformed the NP-hard variable selection step into a sequence of tractable regression subproblems, significantly improving for SPCA have been proposed.In addition, there are many studies on the algorithmic aspects of the SPCA.\cite{Moghaddam2006} formulated Sparse PCA as a combinatorial optimization problem via discrete spectral methods, and proposed greedy (GSPCA) and exact (ESPCA) algorithms to efficiently extract sparse components with improved variance retention and interpretability. \cite{Wang2012} proposed an iterative elimination algorithm for sparse PCA, inspired by recursive feature elimination (RFE) techniques. Their method eliminates variables one by one based on criteria minimizing variance loss—either by absolute loading values or an approximated minimal variance loss (AMVL) metric—resulting in a simple yet effective approach that bypasses complex optimization.

Despite these advances, SPCA still has limitations. The use of the $L_2$ penalty tends to apply uniform shrinkage across all variables, failing to distinguish between those with strong and weak contributions to the variance. This can lead to information redundancy and reduced interpretability. Moreover, in high-dimensional, low-sample-size settings, the estimated loadings can be highly sensitive to sample variation, resulting in unstable principal components.

Inspired by recent work by \cite{Y1989} on single-parameter regularization for ridge regression, which adaptively penalizes variables according to their relative importance, we propose an improvement to the traditional SPCA framework. Specifically, we modify the $L_2$ penalty by introducing a balance parameter that allows for more flexible trade-offs between sparsity and variance retention. The resulting method, termed  Single-Parametric Sparse PCA (SP-SPCA), aims to provide principal components that are both more interpretable and more stable in high-dimensional settings.

\section{Motivation and Method Details}\label{sec2}
\subsection{Single-Parametric Sparse Principal Component Analysis }\label{sec2.1}
Traditional SPCA methods are constructed based on ridge regression estimation. However, ridge regression estimation has a significant drawback: its shrinkage effect on the regression coefficients is uniform, meaning that all covariates are shrunk to the same degree, regardless of their contribution to the response variable. We believe this aspect can be improved. The single-parameter principal component regression estimation proposed by \cite{Y1989}  effectively addresses this issue. Therefore, this work proposes a new sparse principal component analysis  method based on single-parameter principal component regression estimation. The methodological details are as follows.

Let $X$ be a $p-$dimensional random vector with covariance matrix $\Sigma$. Performing the singular value decomposition (SVD) on $X$, we obtain:

 \begin{eqnarray}
	X=UDV^{T}.
\end{eqnarray}

Thus, $X^{T}X=VD^{2}V^{T},$ where $D=diag(d_{1},d_{2},\cdots,d_{p})$,  so the eigenvalues of \( X^TX \) are \( d_i^2 \) for \( i = 1,2,\dots,p \).  

According to the definition of Principal Component Analysis (PCA), the \( i \)-th principal component is given by \(Y_i = X V_i\) with variance \( d_i^2 \).

From the definition of Single-Parameter Principal Component Regression Estimation (SPPCR), we have:

\begin{equation}
	\hat{\beta}^{\text{SPPCR}} = (X^T X + V K V^T)^{-1} X^T Y,
\end{equation}
where 

\begin{equation}
	K = \text{diag} \left( \frac{d_1^2(1 - \theta)}{d_1^2 + \theta - 1}, \frac{d_2^2(1 - \theta)}{d_2^2 + \theta - 1}, \dots, \frac{d_r^2(1 - \theta)}{d_r^2 + \theta - 1},\frac{1}{\theta}-d_{r+1},\cdots,\frac{1}{\theta}-d_{p} \right),
\end{equation}
where \( r \) satisfies \(d_r^2 \geq 1 \geq d_{r+1}^2\) and \(\theta \in (d_p^2, 1).\)

For the relationship between SPPCR estimation and Principal Component Analysis (PCA), we present the following theorem:

\begin{theorem}\label{th1}
	For $\forall i$, denote $Y_{i}=U_{i}D_{ii}$.$Y_{i}$ is the i-th principal components. Suppose $\hat{\beta}^{\text{SPPCR}}$ is the single-parameter principal component regression estimation given by 
	\begin{equation}
		\hat{\beta}^{SPPCR}=\underset{\beta}{\text{argmin}}\{\Vert Y_{i}-X\beta\Vert_{2}^{2}+\Vert\sqrt{K}V^{T}\beta\Vert_{2}^{2}\}.
	\end{equation}
Let $\hat{\nu}=\frac{\hat{\beta}^{SPPCR}}{\Vert\hat{\beta}^{SPPCR}\Vert_{1}}$, then $\hat{\nu}=V_{i}$.	   
\end{theorem}

\noindent\textbf{Proof of Theorem \ref{th1}}
Denote that $X^{T}X=VD^{2}V^{T}$ and $V^{T}V=I$, we have:
\begin{align*}
\hat{\beta}^{SPPCR}&=(X^{T}X+VKV^{T})^{-1}X^{T}XV_{i}\\\nonumber
&=V\frac{D^{2}}{D^{2}+K}V^{T}V_{i}\\\nonumber
&=V_{i}\frac{D_{ii}^{2}}{D_{ii}^{2}+K_{ii}},
\end{align*}
thus $\hat{\nu}=V_{i}$.

Since \( Y_i = U_{i}D_{ii} \), the theorem relies on the results of principal component analysis. Therefore, we consider the following theorem:	

\begin{theorem}\label{th2}
		Let \( X_i \) denote the \( i \)-th row vector of the random vector \( X \), and define \( Z = \sqrt{K}V^T \). Then,
		\begin{equation}
			(\hat{\alpha}, \hat{\beta}) = \underset{\alpha, \beta}{argmin}\{ \Vert X_i - \alpha \beta^T X_i \Vert_2^2 + \Vert Z \beta \Vert_2^2\},
		\end{equation}
		subject to the constraint \( \Vert \alpha \Vert_2^2 = 1 \). It follows that \( \hat{\beta} \propto V_1 \).
\end{theorem}
c\noindent\textbf{Proof of Theorem \ref{th2}}
\begin{align*}
 \sum_{i=1}^{n}\Vert X_{i}-\alpha\beta^{T}X_{i}\Vert_{2}^{2}&=\sum_{i=1}^{n}TrX_{i}^{T}(I-\beta\alpha^{T})(I-\alpha\beta^{T})X_{i}\\\nonumber
 &=\sum_{i=1}^{n}X_{i}X_{i}^{T}Tr(I-\beta\alpha^{T})(I-\alpha\beta^{T})\\\nonumber
 &=Tr(I-\beta\alpha^{T}-\alpha\beta^{T}+\beta\alpha^{T}\alpha\beta^{T})X^{T}X\\
 &=TrX^{T}X+Tr\beta^{T}X^{T}X\beta-2Tr\alpha^{T}X^{T}X\beta \\
 &=TrX^{T}X+\beta^{T}X^{T}X\beta-2\alpha X^{T}X \beta.
\end{align*}
Since $Z=\sqrt{K}V^{T}$, then
\begin{eqnarray*}
	\Vert Z\beta\Vert^{2}_{2}=\beta^{T}Z^{T}Z\beta=\beta^{T}VKV^{T}\beta,
\end{eqnarray*}
thus 
\begin{eqnarray*}
	\sum_{i=1}^{n}\Vert X_{i}-\alpha\beta^{T}X_{i}\Vert_{2}^{2}+\Vert Z\beta\Vert^{2}_{2}=TrX^{T}X-2\alpha X^{T}X \beta+ \beta^{T}(X^{T}X+VKV^{T})\beta.
\end{eqnarray*}
 For a fixed \( \alpha \), the minimum value of the above equation is
 \begin{eqnarray*}
 \beta=(X^{T}X+VKV^{T})^{-1}X^{T}X\alpha,
 \end{eqnarray*}
then we have
\begin{eqnarray}
		\sum_{i=1}^{n}\Vert X_{i}-\alpha\beta^{T}X_{i}\Vert_{2}^{2}+\Vert Z\beta\Vert^{2}_{2}=TrX^{T}X-\alpha X^{T}X(X^{T}X+VKV^{T})^{-1}X^{T}X\alpha.
\end{eqnarray}
Therefore, 
\begin{eqnarray}
	\hat{\alpha}=\underset{\alpha}{\text{argmax}}\{\alpha X^{T}X(X^{T}X+VKV^{T})^{-1}X^{T}X\alpha\}	
\end{eqnarray}
\begin{eqnarray*}
	s.t.\alpha^{T}\alpha=1,
\end{eqnarray*}
denote that $\hat{\beta}=(X^{T}X+VKV^{T})^{-1}X^{T}X\hat{\alpha}$ and $X^{T}X=VD^{2}V^{T}$, we have
 \begin{eqnarray}
 	X^{T}X(X^{T}X+VKV^{T})^{-1}X^{T}X=V\frac{D^{4}}{D^{2}+K}V^{T}.
 \end{eqnarray}
Thus $\hat{\alpha}=sV_{1},s=1 or -1$, then $\hat{\beta}=s\frac{D_{11}^{2}}{D_{11}^{2}+K_{11}}V_{1}$

Theorem \ref{th2} does not rely on the results of principal component analysis but derives the first principal component directly based on regression. Next, we extend the result of Theorem \ref{th2} to the first \( k \) principal components.
	
\begin{theorem}\label{th3}
		Suppose we consider the first \( k \) principal components. Define \( A_{p \times k} = (\alpha_1, \alpha_2, ..., \alpha_k) \) and \( B_{p \times k} = (\beta_1, \beta_2, ..., \beta_p) \), with \( Z = \sqrt{K}V^T \). Then,
		\begin{equation}
			(\hat{A}, \hat{B}) = \underset{A,B}{argmin}\{\sum_{i=1}^{n} \Vert X_i - A B^T X_i \Vert_2^2 + \sum_{j=1}^{k} \Vert Z \beta_j \Vert_2^2\},
		\end{equation}
		subject to the constraint \( \Vert A \Vert_2^2 = I_k \), implying that \( \hat{\beta}_j \propto V_j \) for \( j = 1,2,...,k \).
\end{theorem}
\noindent\textbf{Proof of Theorem \ref{th3}}	
By analogy with the proof of the preceding theorem, we can similarly obtain
 \begin{eqnarray*}
 	\sum_{i=1}^{n} \Vert X_i - A B^T X_i \Vert_2^2 + \sum_{j=1}^{k} \Vert Z \beta_j \Vert_2^2=TrX^{T}X+\sum_{j=1}^{k} (-2\alpha_{j}^{T}X^{T}X\beta_{j}+\beta_{J}^{T}(X^{T}X+VKV^{T})\beta_{j}).
 \end{eqnarray*}
 For a fixed \( A \), the minimum value of the above equation is
\begin{eqnarray*}
	B=(X^{T}X+VKV^{T})^{-1}X^{T}XA,
\end{eqnarray*}
thus 
\begin{eqnarray}
	\hat{A}=\underset{A}{\text{argmax}}\{TrAX^{T}X(X^{T}X+VKV^{T})^{-1}X^{T}XA\},	
\end{eqnarray}
this is an eigenvalue decomposition problem, whose solution is $\hat{\alpha_{j}}=s_{j}V_{j},s_{j}=1 or -1$, then $\hat{\beta_{j}}=s_{j}\frac{D_{jj}^{2}}{D_{jj}^{2}+K_{jj}}V_{j}$ for $j=1,2,\cdots,k$.

	In this paper, we incorporate \( L_1 \) regularization to propose an interpretable and stable framework for sparse principal component analysis, namely the Single-Parametric Sparse Principal Component Analysis (SP-SPCA):

    \begin{equation}
		(\hat{A}, \hat{B}) = \underset{A,B}{argmin}\{\sum_{i=1}^{n} \Vert X_i - A B^T X_i \Vert_2^2 + \sum_{j=1}^{k} \Vert Z \beta_j \Vert_2^2 + \sum_{j=1}^{k} \lambda_j \Vert \beta_j \Vert_1\}.
	\end{equation}
	\begin{eqnarray*}
		s.t.A^{T}A=I_{k},
	\end{eqnarray*}

\subsection{Algorithm for solving SP-SPCA}
We know that $A$ is an orthogonal matrix. Let $A_\perp$ be any matrix whose columns are orthogonal to those of $A$, such that $[A; A_\perp]$ forms a $p \times p$ orthogonal matrix. Therefore, we have:

\begin{align*}
	\sum_{i=1}^{n} \| X_i - AB^T X_i \|_2^2 &= \| X - XBA^T \|_2^2  \\
	&= \| X A_\perp \|_2^2 + \| X A - X B \|_2^2  \\
	&= \| X A_\perp \|_2^2 + \sum_{j=1}^{k} \| X \alpha_j - X \beta_j \|_2^2 
\end{align*}

First, consider the problem of solving $B$ given a fixed $A$. For $j = 1, \cdots, k$, let $Y_j^* = X \alpha_j$. According to the analysis above, we have:

\begin{equation}\label{221}
	(\hat{A}, \hat{B}) = \arg \min_{A, B} \left\{ \| X A_\perp \|_2^2 + \sum_{j=1}^{k} \left[ \| Y_j^* - X \beta_j \|_2^2 +  \Vert Z \beta_j \Vert_2^2 +  \lambda_j \Vert \beta_j \Vert_1 \right] \right\}.
\end{equation}

Since $A$ is fixed, the problem above reduces to:

\begin{equation}\label{eq222}
	\hat{\beta}_j = \arg \min_{\beta_j} \left\{ \| Y_j^* - X \beta_j \|_2^2 + \Vert Z \beta_j \Vert_2^2 +  \lambda_j \Vert \beta_j \Vert_1 \right\}.
\end{equation}

We refer to the above optimization problem as the **SPPCSO problem** (Sparse Principal Component with Structured Optimization). On the other hand, given a fixed $B$, the goal is to minimize $\sum_{i=1}^{n} \| X_i - A B^T X_i \|_2^2 = \| X - X B A^T \|_2^2$ subject to the constraint $A^T A = I$. The solution is given by the following theorem.

\begin{theorem}\label{th2}
	Let $M_{n \times p}$ and $N_{n \times k}$ be two matrices. Consider the following minimization problem:
	
	\[
	\hat{A} = \arg \min_A \| M - N A^T \|_2^2, \quad \text{subject to } A^T A = I_{k \times k}.
	\]
	
	Assume the SVD of $M^T N$ is $U D V^T$. Then the solution is
	
	\[
	\hat{A} = U V^T.
	\]
\end{theorem}

To solve equation(\ref{eq222}), we only need the matrix $X^T X$. The reason is as follows:

\begin{align}\label{223}
	\| Y_j^* - X \beta_j \|_2^2 + \Vert Z \beta_j \Vert_2^2 +  \lambda_j \Vert \beta_j \Vert_1
	&= (\alpha_j - \beta_j)^T X^T X (\alpha_j - \beta_j) + \Vert Z \beta_j \Vert_2^2 +  \lambda_j \Vert \beta_j \Vert_1.
\end{align}

This applies also to computing the SVD in $(X^T X) B = U D V^T$. Note that $\frac{1}{n} X^T X$ is the sample covariance matrix. If the true covariance matrix $\Sigma$ of $X$ is known, we may substitute $\Sigma$ in place of $X^T X$ in~\eqref{223}, leading to the following form:

\begin{align}\label{3.2.4}
	(\alpha_j - \beta_j)^T \Sigma (\alpha_j - \beta_j) + \Vert Z \beta_j \Vert_2^2 +  \lambda_j \Vert \beta_j \Vert_1. 
\end{align}

Although equation~\eqref{3.2.4} is not in the standard SPPCSO form, it can be transformed accordingly. Let

\[
Y^* = \Sigma^{\frac{1}{2}} \alpha_j, \quad X^* = \Sigma^{\frac{1}{2}},
\]

then the optimization becomes:

\[
\hat{\beta}_j = \arg \min_{\beta_j} \left\{ \| Y^* - X^* \beta_j \|_2^2 + \Vert Z \beta_j \Vert_2^2 +  \lambda_j \Vert \beta_j \Vert_1 \right\}.
\]

Hence, the core of the algorithm is to solve the SPPCSO problem. The SPPCSO problem is equivalent to:

\begin{eqnarray}\label{eq11}
	\hat{\beta} := \arg \min_{\beta} \left\{ \| y - X \beta \|_2^2 + \| Z \beta \|_2^2 + \lambda \| \beta \|_1 \right\}.
\end{eqnarray}

We define an augmented dataset $(\tilde{y}, \tilde{X})$ by:

\begin{eqnarray*}
	\tilde{X} = \begin{pmatrix}
		X\\ Z
	\end{pmatrix}, \quad
	\tilde{y} = \begin{pmatrix}
		y\\ 0
	\end{pmatrix},
\end{eqnarray*}

\noindent
Then the SPPCSO problem can be rewritten as:

\begin{eqnarray*}
	\hat{\beta} := \arg \min_{\beta} \left\{ \| \tilde{y} - \tilde{X} \beta \|_2^2 + \lambda \| \beta \|_1 \right\}.
\end{eqnarray*}

This is clearly a standard Lasso problem and can be efficiently solved using coordinate descent. We describe the algorithm in detail below.

\begin{algorithm}
	\caption{Coordinate Descent Algorithm for SPPCSO}
	\label{alg:SPPCSO}
	\begin{algorithmic}[1]
		\Require Initial estimate $\hat{\beta}^{(0)}$, parameters $\lambda$, $\theta$, tolerance $\epsilon = 10^{-4}$
		\State Define the active set $\xi^{(0)} = \{ j \in \{1,\dots,p\} : \hat{\beta}_j^{(0)} \neq 0 \}$
		
		\vspace{0.5em}
		\State Construct the augmented data:
		\[
		X^* = \begin{pmatrix} X \\ Z \end{pmatrix}, \quad
		y^* = \begin{pmatrix} y \\ 0 \end{pmatrix}
		\]
		where $Z = \sqrt{K} U^T$, $K = \mathrm{diag}\left( \frac{d_1(1-\theta)}{d_1 + \theta - 1}, \dots, \frac{d_r(1-\theta)}{d_r + \theta - 1}, \frac{1}{\theta} - d_{r+1}, \dots, \frac{1}{\theta} - d_p \right)$
		
		\State $U$ is the orthogonal matrix from SVD of $X_{\xi^{(0)}}^T X_{\xi^{(0)}}$, and $d_i$ is its $i$-th eigenvalue
		\Ensure Optimized coefficients $\hat{\beta}^{\text{SPPCSO}}$
		
		\vspace{0.5em}
		\State Initialize $\hat{\beta}^{(0)} = \hat{\beta}^{\text{Lasso}}$, set iteration counter $k = 0$
		\Repeat
		\For{$j = 1$ to $p$}
		\State Compute partial residual:
		\[
		r_j = (X_j^*)^T \left( y^* - X_{-j}^* \hat{\beta}^{(k)}_{-j} \right)
		\]
		\State Update:
		\[
		\hat{\beta}_j^{(k+1)} = S(r_j, \lambda), \quad S(r, \lambda) = \text{sign}(r) \cdot \max(|r| - \lambda, 0)
		\]
		\EndFor
		\If{$\|\hat{\beta}^{(k+1)} - \hat{\beta}^{(k)}\| < \epsilon$}
		\State $\hat{\beta}^{\text{SPPCSO}} = \hat{\beta}^{(k+1)}$
		\State \Return $\hat{\beta}^{\text{SPPCSO}}$
		\EndIf
		\State $k \gets k + 1$
		\Until{convergence}
	\end{algorithmic}
\end{algorithm}

\begin{enumerate}
	\item Initialize $A = V[:, 1:k]$, i.e., the first $k$ standard principal components obtained from PCA;

	\item Given a fixed $A = [\alpha_1, \cdots, \alpha_k]$, solve the following $k$ SPPCSO problems:
	\[
	\beta_j = \arg \min_{\beta} \left\{ (\alpha_j - \beta)^T X^T X (\alpha_j - \beta) + \| Z \beta_j \|_2^2 + \lambda_j \| \beta_j \|_1 \right\};
	\]
	
	\item For a fixed $B = [\beta_1, \cdots, \beta_k]$, compute the SVD decomposition:
	\[
	X^T X B = U D V^T,
	\]
	and update $A = U V^T$;
	
	\item Repeat Steps 2 and 3 until convergence;
	
	\item Normalize the sparse loadings: $\tilde{V}_j = \dfrac{\beta_j}{\| \beta_j \|_1}, \quad j = 1, 2, \cdots, k.$
\end{enumerate}

\section{Simulations}\label{sec4}

\subsection{Preliminary Validation on Low-Dimensional Data}

Similar to the numerical simulations developed by \cite{Z2006} in SPCA, we first  verify the performance of SP-SPCA under a low-dimensional data structure. We generate the following three variables:

\begin{equation}
	\begin{aligned}
		&V_1 \sim N(0, 50), \\
		&V_2 \sim N(0, 300),\\
		&V_3 = -0.3V_1 + 0.925V_2 + \varepsilon,\quad \varepsilon \sim N(0, 1).
	\end{aligned}
\end{equation}

$V_1$, $V_2$, and $V_3$ are independent of each other. We give the following 10 observed variables:
\begin{equation}
	\begin{aligned}
		&X_i = V_1 + \varepsilon_i^1,\quad \varepsilon_i^1 \sim N(0,1),\quad i = 1, 2, 3, 4, \\
		&X_i = V_2 + \varepsilon_i^2,\quad \varepsilon_i^2 \sim N(0,1),\quad i = 5, 6, 7, 8,\\
		&X_i = V_3 + \varepsilon_i^3,\quad \varepsilon^3 \sim N(0,1),\quad i = 9, 10,
	\end{aligned}
\end{equation}
where $\varepsilon_i^j,\quad j=1,2,3,\quad i=1,2,\ldots,10$ are independent of each other. We use the covariance matrix of $(X_1, \ldots, X_{10})$ to compute PCA, SPCA, SP-SPCA.

\begin{table}[htbp]
	\tabcolsep 0pt \vspace*{-12pt}
	\renewcommand{\arraystretch}{1.3}
	\def\temptablewidth{1.0\textwidth}
	\centering
	\caption{Loading matrix and explained variance for PCA, SPCA, and SP-SPCA on a low-dimensional dataset.}
	\label{tab:lowdim_tabularstar}
	\begin{tabular*}{\temptablewidth}{@{\extracolsep{\fill}} lcc|cc|cc}
		\hline\hline
		\textbf{} & \multicolumn{2}{c|}{\textbf{PCA}} & \multicolumn{2}{c|}{\textbf{SPCA}} & \multicolumn{2}{c}{\textbf{SP-SPCA}} \\
		\cline{2-7}
		& PC1 & PC2 & PC1 & PC2 & PC1 & PC2 \\
		\hline
	$X_1$ & -0.315 & 0.400 & 0   & 0.590& 0   & 0.491 \\
	$X_2$ & -0.315 & 0.349 & 0   & 0.206 & 0   & 0.501 \\
	$X_3$ & -0.315 & 0.404 & 0   &  0.566 & 0   & 0.502 \\
	$X_4$ & -0.315 & 0.396 & 0   &  0.538 & 0   & 0.498 \\
	$X_5$ &  0.317 & 0.277 & 0.388 & 0    &  0.499 & 0   \\
	$X_6$ &  0.317 & 0.225 & 0.651 & 0    &  0.501 & 0   \\
	$X_7$ &  0.317 & 0.274 & 0.336 & 0    &  0.502 & 0   \\
	$X_8$ &  0.317 & 0.276 & 0.559& 0    &  0.499 & 0   \\
	$X_9$ &  0.317 & 0.218 & 0     & 0    & 0     & 0   \\
	$X_{10}$ &  0.317 & 0.221 & 0  & 0    & 0     & 0   \\
		\hline
		\textbf{Variance} & 97.6\% & 1.44\% & 60.55\% & 10.30\% & 64.5\% & 11.1\% \\
		\hline\hline
	\end{tabular*}
\end{table}

The number of observed variables associated with the three basic variables $V_1$, $V_2$, and $V_3$ are 4, 4, and 2, respectively, thus $V_1$ and $V_2$ are of comparable significance, while the interpretability of $V_3$ is relatively low. In the principal component analysis, the cumulative variance contribution of the first two principal components has reached 99.04\%, so we only need to focus on the orthogonal sparse representation of the first two principal components, which is consistent with the number of $V_1$ and $V_2$ basic variables.

Since the variance of $V_2$ is significantly higher than that of $V_1$, ideally, the first principal component should be represented by $X_5, X_6, X_7, X_8$, i.e., the variables originate from $V_2$; and the second principal component should be represented by $X_1, X_2, X_3, X_4$, which corresponds to $V_1$. Their loading coefficients should be uniformly distributed and close in value. Since $X_9$ and $X_{10}$ are derived from $V_3$, which is a combination of $V_1$ and $V_2$ plus independent noise, their statistical properties are closer to those of $V_1$ and $V_2$, and thus their contributions to the principal components should be close to those of $X_1$ to $X_8$. In addition, the variable loading coefficients on $V_3$ should be sparsified in the principal component representation. In the results, SP-SPCA achieves feature selection consistent with SPCA, but its resulting loading matrix is closer to the ideal, with a cumulative variance explained of 75.6\%, higher than SPCA’s 70.8\%, thus retaining more information.

In summary, SP-SPCA is effective in variable screening, effectively filtering out noisy variables, while enhancing the explanatory power of variables and retaining higher information. This also indicates that SP-SPCA is more advantageous than traditional SPCA in dealing with redundant and noisy variables. Next, in order to further test the advantages of SP-SPCA in principal component analysis, we apply the model to high-dimensional datasets for simulation experiments.

\subsection{High-dimensional simulation data generation process and simulation settings}

 To evaluate the performance of the proposed SP-SPCA method for high-dimensional sparse structure identification, we designed a set of simulation experiments based on the latent factor model. The key parameter settings involved in the simulation experiments are shown in Table \ref{tab:parameter_settings}.
 
 \begin{table}[htbp]
 	\tabcolsep 0pt \vspace*{-12pt}
 	\renewcommand{\arraystretch}{1.3}
 	\def\temptablewidth{0.6\textwidth}
 	\centering
 	\caption{Key parameter settings used in the simulation studies}
 	\label{tab:parameter_settings}
 	\begin{tabular*}{\temptablewidth}{@{\extracolsep{\fill}} lc}
 		\hline\hline
 		\textbf{Parameter} & \textbf{Symbol}  \\
 		\hline
 		Sample size                & $n$   \\
 		Total number of variables  & $p$  \\
 		Number of latent factors   & $d$  \\
 		Variables per factor       & $s$   \\
 		Number of mixed variables  & $m$   \\
 		\hline\hline
 	\end{tabular*}
 \end{table}
\noindent
The specific data generation process is as follows:

\begin{itemize}
	\item We split the variables into structural variables $\mathbf{X_{struct}}$ and mixed variables $\mathbf {X_{mix}}$. First, generate the latent factor matrix $F \in \mathbb{R}^{n \times d}$ with elements obeying the standard normal distribution;
	\item Each factor $F_j,j=1,2,\cdots,d$ corresponds to the control of $s_j$ structural variables, generating a matrix of coefficients $P \in \mathbb{R}^{P_{struct} \times d}$ and superimposed on the standard normal noise to obtain the structural variable part: $\mathbf{X}_{\text{ struct}} = \mathbf{F} \mathbf{P}^\top + \varepsilon$,$\varepsilon\sim N(0,1)$;
	\item For the mixed variable part, each variable is jointly influenced by $2$ to $3$ randomly selected latent factors, with coefficients randomly generated and unit normalized, and superimposed noise to obtain $\mathbf{X}_{\text{mix}}$;
	\item Splice the above structural and mix variables to form the final data matrix $\mathbf{X}=(\mathbf{X_{struct}},\mathbf {X_{mix}})$ and center all variables.
\end{itemize}

This simulation structure can effectively simulate the existence of multiple latent factors and cross-influence between variables in real high-dimensional data, so as to verify the extraction ability and sparsity control effect of the model under different structural complexity.

In order to evaluate the performance of the model in an all-round way, we set up the following four groups of simulation experiments for different parameters:

\textit{Case 1.} \ 
We first consider a typical high-dimensional setting: the number of samples is $n = 100$ and the number of variables is $p = 300$. The number of latent factors is set to $d = 5$ and each factor controls $s = 50$ variables. In addition, $m = 50$ mixing variables are included.

\textit{Case 2.} \ 
We investigate the performance stability of SP-SPCA by increasing the variable dimension $p$ while fixing $n = 100$ and $d = 5$. We set $p = 150, 300, 450, 600, 750, 900$ while fixing the number of mixing variables to $m = p/6$, corresponding to $m = 25, 50, 75, 100, 125, 150$, and setting $s = 25, 50, 75, 100, 125, 150$. The case is designed to evaluate the stability of the method in high and ultra-high dimensional situations.

\textit{Case 3.} \ 
We evaluate the robustness of the method to crossover variables by varying the number of mixing variables $m$ while fixing $n = 100$, $p = 300$ and $d = 5$. We consider $m = 10, 50, 100, 150$, corresponding to the setting $s = 58, 50, 40, 30$. This setting is intended to test the resistance of the method to overlapping disturbances with mixed variables.

\textit{Case 4.} \ 
We investigate the ability of the method to parse complex factor structures by varying the number of latent factors $d$. The sample size is fixed at $n = 100$ and the total number of variables is fixed at $p = 300$. We consider $d = 2, 5, 10$ and the number of mixing variables is fixed at $m = 50$, adjusting $s = 125, 50, 25$ accordingly. The experiment is designed to evaluate the structure-resolving ability of SP-SPCA.

To ensure that the results are not oh-so-accidental, we repeated the experiment 50 times for each set of cases, and finally averaged the resulting cumulative variance contributions. These four sets of simulation experiments systematically evaluate the performance and robustness of the proposed method in multiple dimensions. Case 1 provides a benchmark evaluation under standard settings to verify the feasibility of the model in the underlying high-dimensional structure; Case 2 explores the stability and scalability of the method as the dimensionality gradually increases, reflecting its adaptability to ultra-high-dimensional data; Case 3 tests the robustness of the model in the presence of a non-ideal data structure by adjusting the degree of mixed-variable interferences; and Case 4 focuses on the variation in the number of latent factors to test the performance of the method in the presence of complex structures. Case 4 focuses on the variation in the number of latent factors to test the ability of the method in a complex structural discrimination task.

\subsection{Analysis of results of high-dimensional simulation data}

In the study of \cite{Shen2008} et al. it was pointed out that the cumulative variance contribution of principal component results from SPCA is greatly reduced in high dimensional data scenarios, which is due to the fact that SPCA over-penalizes the loadings, resulting in the principal component directions being too sparse at the expense of the variance explanation ability. By empirically summarizing a large number of simulations in the literature, we found that the cumulative variance contribution of SPCA is usually 50\%~65\% in medium and high dimensional data structures, while the cumulative variance contribution is below 50\% in high and extreme high dimensional cases. In this simulation experiment, since a specific data structure was predefined, the numerical comparison focuses on evaluating the cumulative variance explained by each method when producing the loading matrix closest to the ideal sparse structure. According to the previously described data setup, the ideal loading pattern is as follows: for each principal component, the variables $X_{i+1}–X_{i+50},\ i=0,50,100,150,200$ have nonzero loadings, while the remaining variables have zero loadings, with variables $X_{251}–X_{300}$ fully sparse (all loadings zero), as shown in the table \ref{tab:case1_ideal_loadings}. Therefore, we constrain the number of nonzero loadings per principal component to 50.

\begin{table}[htbp]
	\tabcolsep 0pt \vspace*{-12pt}
	\renewcommand{\arraystretch}{1.3}
	\def\temptablewidth{1.0\textwidth}
	\centering
	\caption{Case 1: Ideal Loadings Matrix Results.}
	\label{tab:case1_ideal_loadings}
	\begin{tabular*}{\temptablewidth}{@{\extracolsep{\fill}} cccccc}
		\hline\hline
		& PC1 & PC2 & PC3 & PC4 & PC5 \\
		\hline
		$X_{1}–X_{50}$ & Nonzero & 0 & 0 & 0 & 0 \\
		$X_{51}–X_{100}$ & 0 & Nonzero & 0 & 0 & 0 \\
		$X_{101}–X_{150}$ & 0 & 0 & 0 & Nonzero & 0 \\
		$X_{151}–X_{200}$ & 0 & 0 & 0 & 0 & Nonzero \\
		$X_{201}–X_{250}$ & 0 & 0 & Nonzero & 0 & 0 \\
		$X_{251}–X_{300}$ & 0 & 0 & 0 & 0 & 0 \\
		\hline\hline
	\end{tabular*}
\end{table}

\begin{figure}[htbp]
	\centering
	\begin{minipage}{0.49\textwidth}
		\centering
		\includegraphics[width=\textwidth]{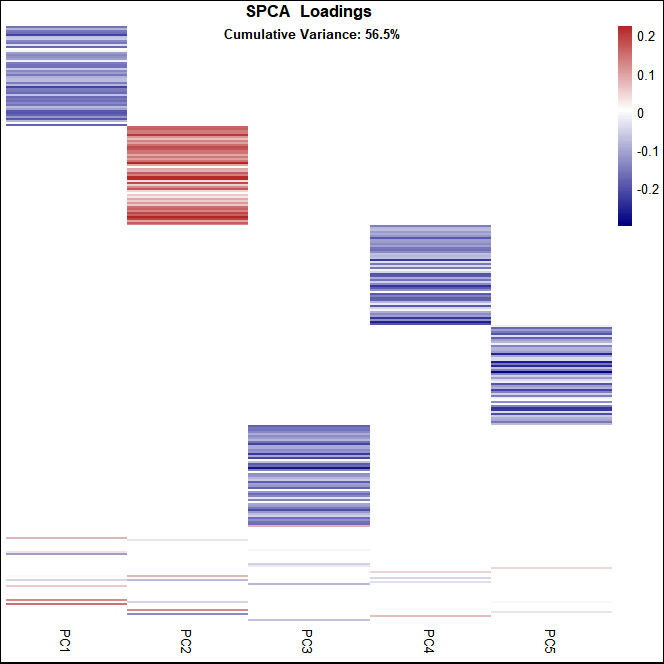}
		\caption{ loadings matrix   under SPCA method in Case 1}
		\label{Fig1}
	\end{minipage}\hfill
	\begin{minipage}{0.49\textwidth}
		\centering
		\includegraphics[width=\textwidth]{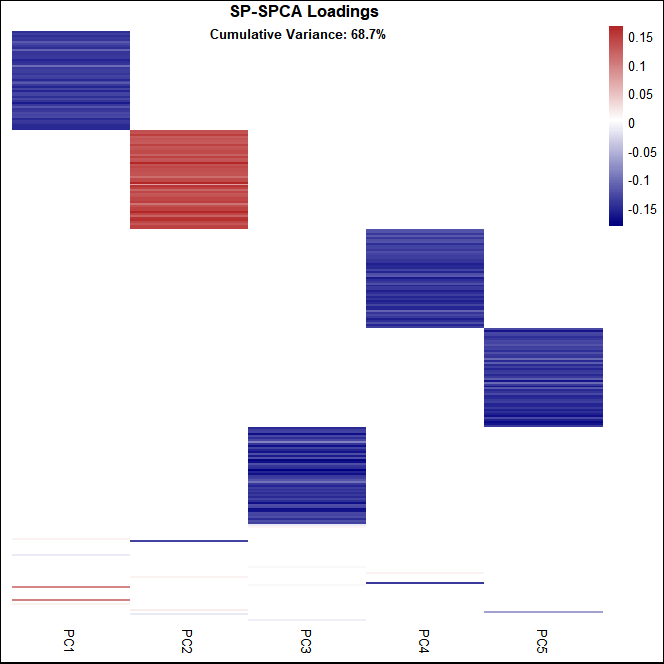}
		\caption{loadings matrix   under SP-SPCA method in Case 1 }
		\label{Fig2}
	\end{minipage}
\end{figure}

Figure \ref{Fig1} and \ref{Fig2}   present a heatmap comparison of the loading matrices obtained by the traditional SPCA and the improved SP-SPCA in Case 1. In the variable blocks $X_{i+1}–X_{i+50},\ i=0,50,100,150,200$, the SPCA heatmap shows large white areas (zeros), whereas SP-SPCA displays concentrated blue/red stripes. Moreover, in the variable block $X_{251}–X_{300}$, which should ideally be fully sparse, SPCA yields more nonzero loadings. As previously described, we constrained the number of nonzero loadings per principal component to 50; thus, these results indicate that SPCA is more prone to erroneous sparsity. Additionally, SP-SPCA achieves a higher cumulative variance explained (68.7\%) compared to SPCA (56.5\%), reducing information loss by 12.2\%. Overall, SP-SPCA maintains sparsity while accurately identifying key variables and effectively eliminating unimportant ones, whereas traditional SPCA tends to over-sparsify, potentially discarding important information (e.g., the nonzero loading of $X_{153}$ was incorrectly shrunk to zero).

\begin{table}[htbp]
	\tabcolsep 0pt \vspace*{-12pt}
	\renewcommand{\arraystretch}{1.3}
	\def\temptablewidth{1.0\textwidth}
	\centering
	\caption{Case 2: Comparison of cumulative variance explained (\%) under different dimensions $p$.}
	\label{tab:case2_dimensional_variance}
	\begin{tabular*}{\temptablewidth}{@{\extracolsep{\fill}} ccccccc}
		\hline\hline
		$p$             & 150   & 300   & 450   & 600   & 750   & 900   \\
		\hline
		SPCA            & 57\%  & 56.5\% & 29.2\% & 20\%  & 16.3\% & 14.4\% \\
		SP-SPCA         & 68.8\% & 65.7\% & 60\%   & 56.9\% & 55.9\% & 54.6\% \\
		\hline\hline
	\end{tabular*}
\end{table}

Based on the comparison results in Table \ref{tab:case2_dimensional_variance} and Figure \ref{simulation}, when the data dimension $p$ increases from 150 to 900 under identical parameter settings, SPCA’s cumulative variance explained drops sharply from a maximum of 57\% to 14.4\% (a 74.8\% decrease), indicating severe information loss under high-dimensional noise. In contrast, SP-SPCA consistently maintains a contribution rate above 54.6\% (peaking at 68.8\%), with only a 20.9\% decrease, retaining 54.6\% of the variance even at $p = 900$. These results further demonstrate the effectiveness and stability of our method for dimensionality reduction in high-dimensional data scenarios.

\begin{table}[htbp]
	\tabcolsep 0pt \vspace*{-12pt}
	\renewcommand{\arraystretch}{1.3}
	\def\temptablewidth{1.0\textwidth}
	\centering
	\caption{Case 3: Comparison of cumulative variance explained (\%) under different numbers of mixed variables $m$ with $n=100$, $p=300$. }
	\label{tab:case3_mixed_variable_variance}
	\begin{tabular*}{\temptablewidth}{@{\extracolsep{\fill}} ccccc}
		\hline\hline
		$m$             & 10    & 50     & 100    & 150    \\
		\hline
		SPCA            & 61\%  & 56.5\% & 38.1\% & 27.2\% \\
		SP-SPCA         & 71.2\%& 65.7\% & 47.3\% & 35.3\% \\
		\hline\hline
	\end{tabular*}
\end{table}

\begin{table}[htbp]
	\tabcolsep 0pt \vspace*{-12pt}
	\renewcommand{\arraystretch}{1.3}
	\def\temptablewidth{1.0\textwidth}
	\centering
	\caption{Case 4: Comparison of cumulative variance explained (\%) under different numbers of latent factors $d$ with $n=100$, $p=300$. }
	\label{tab:case4_factor_variance}
	\begin{tabular*}{\temptablewidth}{@{\extracolsep{\fill}} cccc}
		\hline\hline
		$d$             & 2     & 5      & 10     \\
		\hline
		SPCA            & 22.4\%& 56.5\% & 44\%   \\
		SP-SPCA         & 60\%  & 65.7\% & 61.1\% \\
		\hline\hline
	\end{tabular*}
\end{table}

To further assess the robustness of the methods, we gradually increased the proportion of mixed variables in the data. Mixed variables, being expressed by multiple latent factors, should not define the directions of principal components and are expected to be sparsified by the model. Therefore, increasing the proportion of mixed variables effectively raises the noise level in the data. 

As shown in Table \ref{tab:case3_mixed_variable_variance}, the cumulative variance explained by SPCA decreases sharply from 61\% to 27.2\% as $m$ increases from 10 to 150, reflecting its difficulty in distinguishing signal from noise under increasing noise conditions. In contrast, SP-SPCA, while also showing a decline from 71.2\% to 35.3\%, consistently outperforms SPCA across all levels of $m$ and retains substantial information even when $m = 150$. This indicates that SP-SPCA is more robust to noise and better at maintaining essential signal components in high-noise environments. Moreover, the performance gap between SPCA and SP-SPCA widens as $m$ increases, further highlighting the superior stability and discriminative power of our proposed method in identifying relevant variables and suppressing irrelevant noise. These results demonstrate the effectiveness of SP-SPCA in high-dimensional noisy data settings and underscore its potential for applications requiring reliable variable selection and dimensionality reduction under challenging conditions.

\begin{figure}[!htbp]
	
	\centering
	\subfloat{\includegraphics[width=1\textwidth]{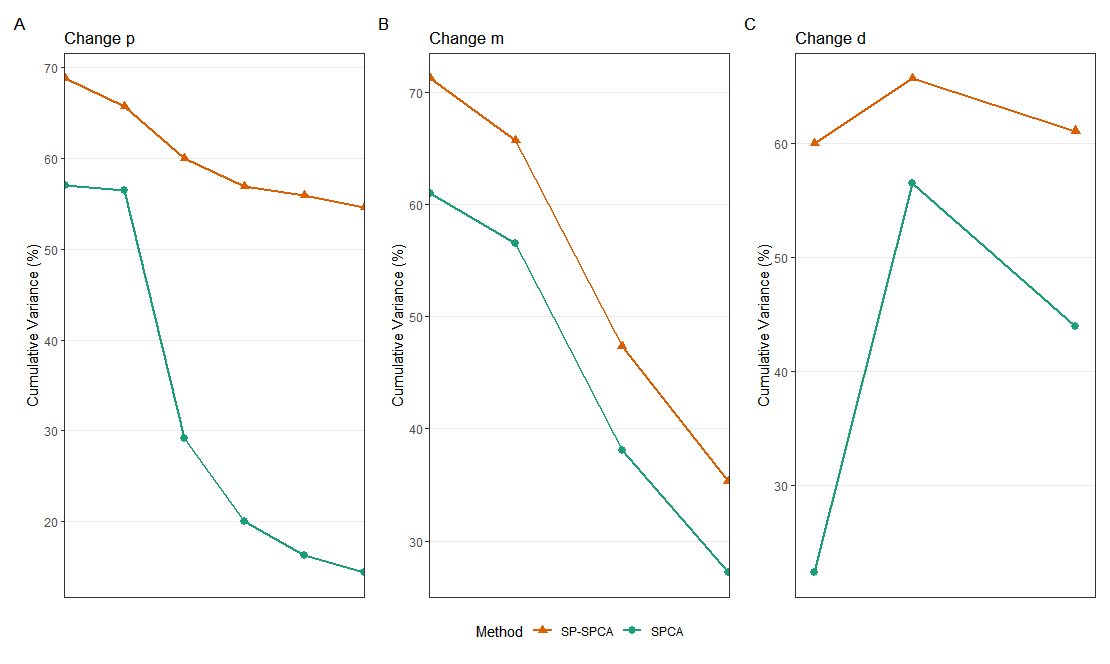}}
	
	\caption{Comparison of numerical simulation results for high-dimensional data.
	}
	\label{simulation}
\end{figure}

Since the number of latent factors $d$ reflects the structural complexity of the data, increasing $d$ effectively challenges the model’s factor resolution ability. As shown in Table \ref{tab:case4_factor_variance}, SP-SPCA consistently maintains a cumulative variance explained above 60\%, demonstrating its robustness in accurately identifying principal component directions and achieving precise sparsity in the loading matrix, even under complex data structures. In contrast, SPCA exhibits significant instability, with the cumulative variance explained dropping sharply when $d$ increases or fluctuating unpredictably, highlighting its limitations in handling high structural complexity.

The above results suggest that SP-SPCA not only performs better in variable selection under noisy or high-dimensional conditions but also adapts more effectively to intrinsic structural complexities of the data. One possible reason is that the single-parameter optimization framework enables SP-SPCA to balance sparsity and signal preservation more effectively, whereas SPCA may suffer from over-sparsification or underfitting when the factor structure becomes intricate. These findings imply that SP-SPCA has stronger generalization potential across diverse data scenarios and is better suited for practical applications where latent structures are often complex and unknown.

\section{Empirical analysis}\label{sec5}
To evaluate the performance of the proposed method on real-world data and cover as many data types as possible, we conducted empirical analyses on two representative datasets, one low-dimensional and one high-dimensional. The details of these datasets are provided in Table \ref{data}.

\begin{table}[htbp]
	\centering
	\caption{Description of datasets used for empirical analysis}
	\label{data}
	\begin{tabular*}{\textwidth}{@{\extracolsep{\fill}} lrrlll}
		\toprule
		 Name & Samples  & Features  & Type & Description & Source \\
		\midrule
		Crime        & 1994        & 99            & Regular        & Social data  & UCI     \\
		S\&P 500     & 150         & 503           & High-dimensional & Financial data & Kaggle  \\
		\bottomrule
	\end{tabular*}
\end{table}

\subsection{Crime data}
We first selected the low-dimensional Crime dataset as the initial analysis target. This dataset contains 1,994 observations and 99 variables, making it a representative example of low-dimensional data.

\begin{figure}[htbp]
	\centering
	\begin{minipage}{0.49\textwidth}
		\centering
		\includegraphics[width=\textwidth]{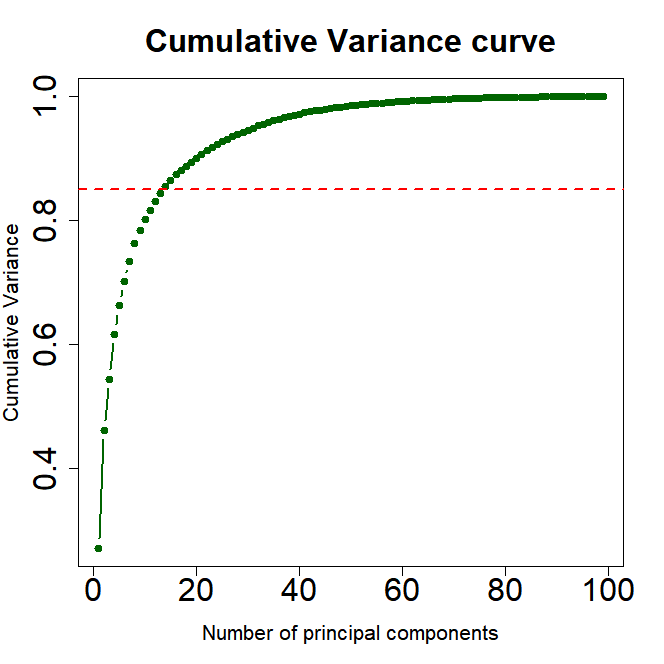}

	\end{minipage}\hfill
	\begin{minipage}{0.49\textwidth}
		\centering
		\includegraphics[width=\textwidth]{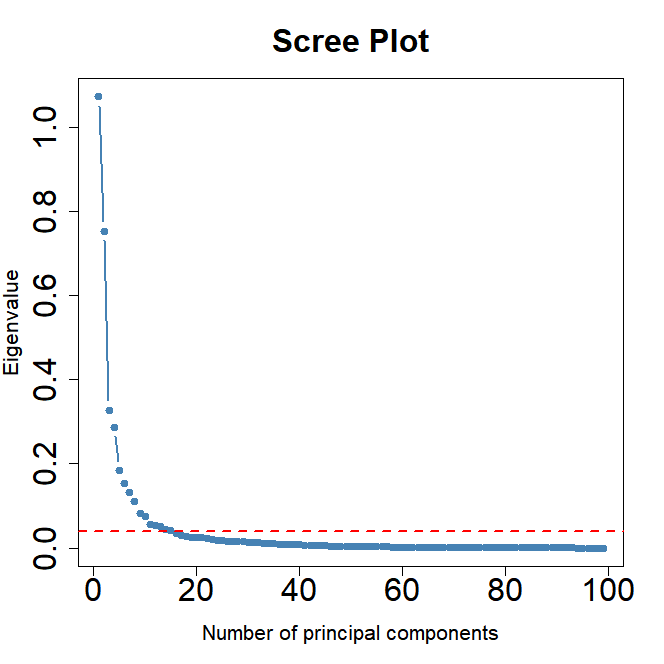}

	\end{minipage}
	\caption{Cumulative explanatory rate curves and gravel plots for the Crime data.}
	    \label{Fig4}
\end{figure}

\begin{table}[htbp]
	\centering
	\caption{Comparison of the total number of nonzero loadings under different cumulative variance explained levels for SPCA and SP-SPCA on the Crime dataset}
	\label{tab:crime_results}
	\begin{tabular*}{\textwidth}{@{\extracolsep{\fill}} ccccccccc}
		\toprule
		 Variance  & 10\% & 20\% & 30\% & 40\% & 50\% & 60\% & 70\% & 80\% \\
		\midrule
		SPCA       & 8   & 27  & 45  & 123  & 220  & 432  & 666  & 1485 \\
		SP-SPCA    & 14  & 24  & 40  & 59   & 104  & 347  & 644  & 1485 \\
		\bottomrule
	\end{tabular*}
\end{table}

Since the true underlying principal component structure of empirical data is typically unknown, our goal is to explain a substantial portion of the variance while obtaining a sparser and more interpretable principal component loading matrix. Therefore, the empirical analysis focuses on comparing the two methods in terms of the total number of nonzero loadings required to achieve the same cumulative variance explained. To determine the appropriate number of principal components, we combined the classical scree plot and the PCA cumulative variance plot.

As shown in Figure \ref{Fig4}, when the number of principal components reaches 15, the cumulative variance explained by PCA reaches 85\%, and the scree plot flattens beyond 15 components. Based on this, we selected the top 15 principal components for analysis. Given that the dataset contains 99 variables, the maximum possible number of nonzero loadings is 1,485.

\begin{figure}[!htbp]
	
	\centering
	\subfloat{\includegraphics[width=1\textwidth]{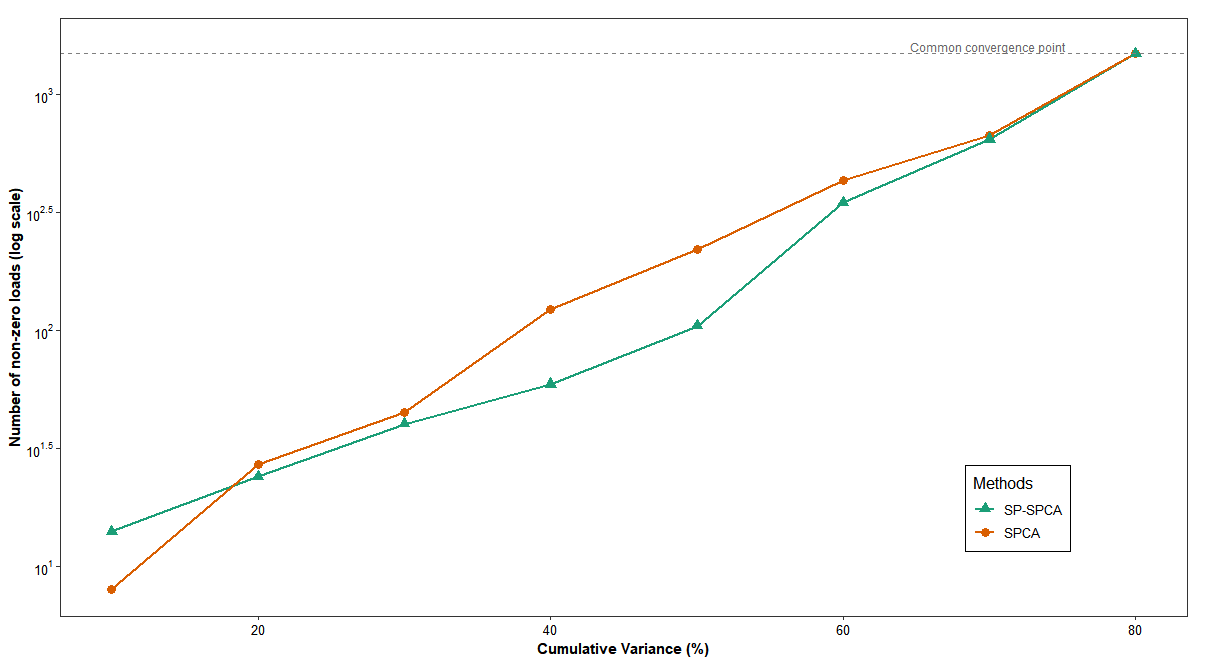}}
	
	\caption{Comparison of results for the Crime data .
	}
	\label{crime}
\end{figure}

Table \ref{tab:crime_results} and Figure \ref{crime} compare the performance of SPCA and SP-SPCA on the Crime dataset. The results show that, at the same cumulative variance explained, SP-SPCA consistently yields fewer nonzero loadings than SPCA, especially in the middle range of 40\%–60\% cumulative variance, where the gap is most pronounced. Specifically, to achieve 40\% cumulative variance, SPCA produces 123 nonzero loadings, while SP-SPCA requires only 59, nearly half. This sparsity advantage persists at 50\% and 60\% explained variance, with SPCA having 220 and 432 nonzero loadings, respectively, compared to only 104 and 347 for SP-SPCA, reflecting a significant reduction in dimensional burden. This indicates that SP-SPCA is more efficient in model compression and variable selection, preserving essential explanatory power while substantially reducing redundancy and enhancing the interpretability and usability of the principal components.

Furthermore, as the cumulative variance approaches 80\%, the number of nonzero loadings between the two methods converges (both around 1,485), indicating that near-complete recovery requires nearly full loadings for both. However, in practical applications, researchers and analysts are typically more interested in the middle range, where a balance between parsimony and explanatory power is crucial. In these key regions, SP-SPCA’s demonstrated sparsity advantage and superior ability to capture principal component structures highlight its superiority in sparse principal component analysis tasks.

\subsection{S$\&$P500 data}

For the high-dimensional data analysis, we selected the S\&P500 index and its constituent stocks’ closing prices from July 2023 to February 2024, resulting in a dataset with 150 observations and 503 variables, classifying it as high-dimensional. However, it is well known that stock prices typically follow a unit root process, exhibiting a random walk pattern over time, which leads to variance increasing with time. This violates the stationarity assumptions of classical statistical models and can distort the covariance matrix estimation in principal component analysis (PCA), causing SPCA to extract spurious factors. In contrast, log returns, defined as $r_t = \ln(P_t / P_{t-1})$, generally satisfy weak stationarity (i.e., constant mean and variance over time) and are more suitable for time series modeling. Therefore, consistent with standard practices in financial research, we apply sparse principal component analysis (SPCA) on the log-return-transformed data.

\begin{figure}[htbp]
	\centering
	\begin{minipage}{0.49\textwidth}
		\centering
		\includegraphics[width=\textwidth]{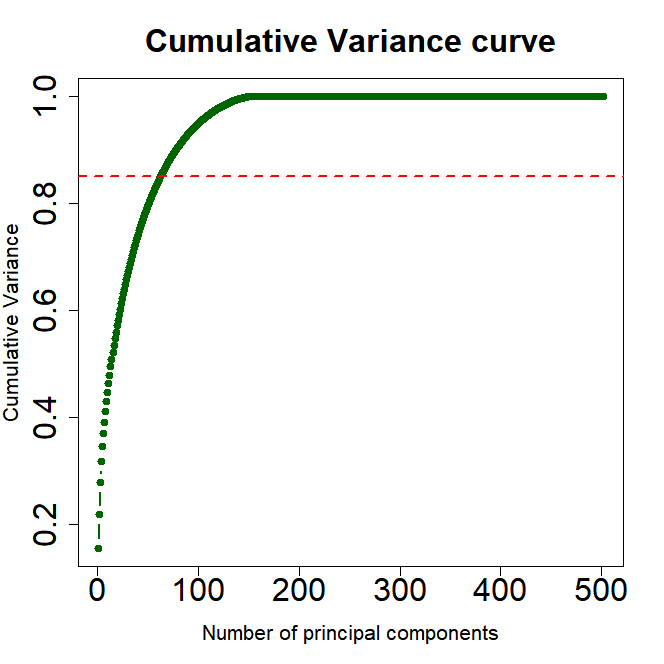}

	\end{minipage}\hfill
	\begin{minipage}{0.49\textwidth}
		\centering
		\includegraphics[width=\textwidth]{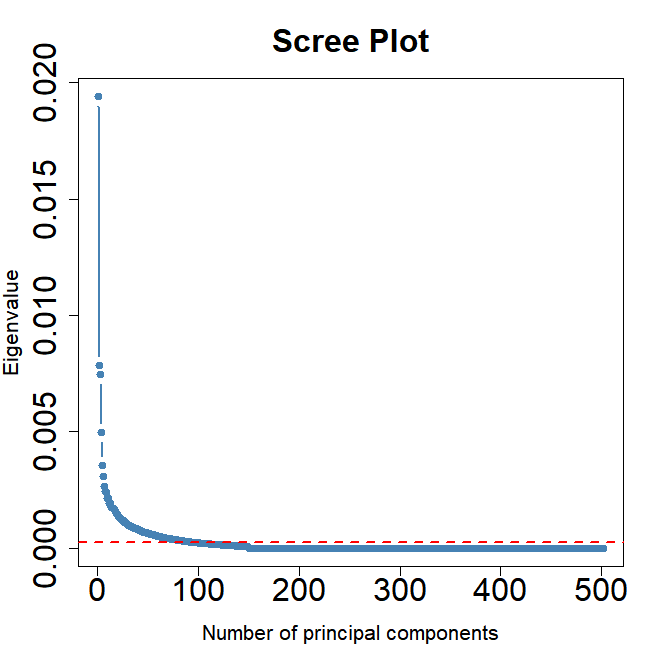}

	\end{minipage}
	\caption{Cumulative explanatory rate curves and gravel plots for the S\&P500 data.}
	\label{Fig5}
\end{figure}
\begin{table}[htbp]
	\centering
	\caption{Comparison of the total number of nonzero loadings under different cumulative variance explained levels for SPCA and SP-SPCA}
	\label{tab:sp500_results}
	\begin{tabular*}{\textwidth}{@{\extracolsep{\fill}} ccccccccc}
		\toprule
		Variance & 10\% & 20\% & 30\% & 40\% & 50\% & 60\% & 70\% & 80\% \\
		\midrule
		SPCA       & 16  & 75   & 209  & 2187  & 6319  & 9852  & 12856 & 25150 \\
		SP-SPCA    & 16  & 52   & 152  & 2002  & 5721  & 9055  & 12103 & 25150 \\
		\bottomrule
	\end{tabular*}
\end{table}
As shown in Figure \ref{Fig5}, when the number of principal components reaches 80, the cumulative variance explained by PCA reaches 85\%, and the scree plot flattens beyond 80 components. However, considering sparsity, we chose to extract the top 50 principal components for further analysis. Meanwhile, we note that this dataset contains 503 variables, and with 50 principal components selected, the maximum possible number of nonzero loadings is 25,150.

Consistent with the previous analysis, we compare the two methods based on the total number of nonzero loadings required to achieve the same cumulative variance explained. The number of principal components is determined using a combination of the classical scree plot and the PCA cumulative variance explained curve. Notably, this dataset contains 503 variables, and with 50 principal components selected, the maximum possible number of nonzero loadings is 25,150.

\begin{figure}[!htbp]
	
	\centering
	\subfloat{\includegraphics[width=1\textwidth]{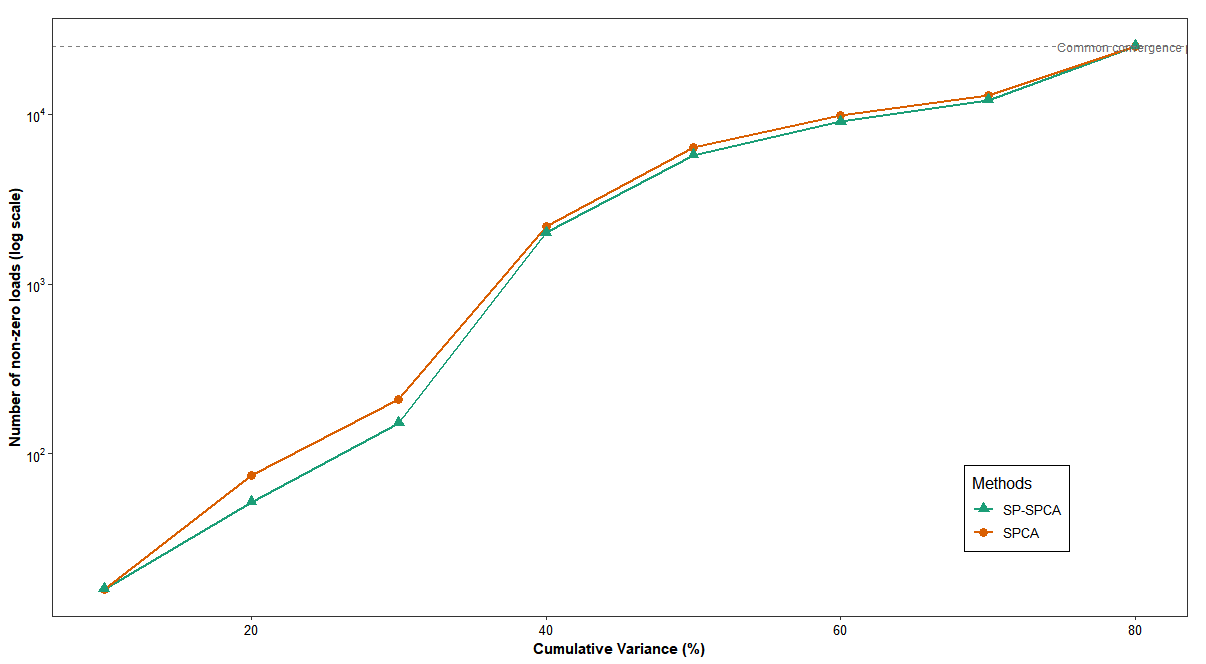}}
	
	\caption{Comparison of results for the S\&P500 data .
	}
	\label{sp500}
\end{figure}

As shown in Table \ref{tab:sp500_results} and Figure \ref{sp500}, the empirical analysis on the high-dimensional S\&P 500 financial dataset demonstrates the significant sparsity advantage of the SP-SPCA method. Specifically, within the 10\%–70\% cumulative variance explained range, SP-SPCA consistently yields fewer nonzero loadings than the traditional SPCA method, with relative reductions ranging from 17.3\% to 27.3\%. For example, at the 20\% variance level, SP-SPCA produces 52 nonzero loadings compared to SPCA’s 75, a reduction of 30.7\%. This gap is particularly pronounced at moderate explanation levels (40\%–60\%); for instance, at 50\% cumulative variance, SP-SPCA’s 5,721 nonzero loadings are 9.5\% fewer than SPCA’s 6,319, indicating that SP-SPCA more effectively eliminates redundant variable associations while preserving the main covariance structure.

Additionally, in the low-variance phase ($<30\%$), SP-SPCA exhibits a steeper sparsity growth curve, with a slope 38.6\% lower than SPCA in the 20\%–30\% interval, suggesting that it captures the same level of variance using fewer variables. In the high-variance phase ($>70\%$), although both methods eventually converge to the same maximum number of loadings (25,150), SP-SPCA’s convergence is more gradual, with a growth rate of 107.8\% versus SPCA’s 95.7\% in the 70\%–80\% interval, revealing its superior sparsity control even under high variance explanation demands. Overall, SP-SPCA achieves comparable or even better cumulative variance explanation with fewer nonzero loadings, making it more conducive to simplifying subsequent analyses, improving model robustness, and enhancing feature selection and interpretability, thus providing a superior tool for high-dimensional data analysis.

\section{Conclusion}\label{sec6}

Building upon the SPPCSO framework, this chapter proposes a single-parameter sparse principal component analysis method (SP-SPCA) and systematically investigates its performance and advantages. Through extensive numerical simulations and empirical analyses on high-dimensional real-world datasets, the proposed SP-SPCA is comprehensively compared with the traditional SPCA in terms of variable selection, information retention, and robustness.

In simulated settings with clear low-dimensional structures, the results show that SP-SPCA can accurately identify block structures of variables generated by different latent factors. While maintaining a high cumulative explained variance, SP-SPCA produces sparse loading representations that are more consistent with theoretical expectations. Compared with traditional SPCA, SP-SPCA demonstrates greater stability in avoiding the misselection of irrelevant or weakly correlated variables and effectively mitigates the loss of important information caused by excessive sparsification.

In high-dimensional simulation experiments, as the number of variables, noise proportion, and latent factors gradually increase, the cumulative explained variance of SPCA decreases significantly and exhibits considerable fluctuations. In contrast, SP-SPCA consistently maintains a higher and more stable explanatory capacity, indicating stronger robustness and generalization ability in high-dimensional noisy environments. These findings suggest that SP-SPCA provides a better balance between sparsity and information preservation, enabling more effective discrimination between signal variables and noise variables.

Empirical analyses on the Crime dataset and the S\&P500 high-dimensional financial dataset further support these observations. Under the same level of cumulative explained variance, SP-SPCA requires significantly fewer nonzero loadings than SPCA, with the difference being particularly pronounced in the moderate explanation range. This indicates that SP-SPCA can capture the primary covariance structure of the data with lower model complexity, thereby improving the interpretability and practical usability of the principal components. Moreover, after stationarizing the financial time series, SP-SPCA continues to exhibit strong sparsity control and efficient information extraction, demonstrating its potential applicability in real high-dimensional data analysis.

Overall, the results indicate that SP-SPCA outperforms traditional SPCA in terms of variable selection accuracy, information retention, and robustness in high-dimensional settings. The proposed method therefore provides a more effective approach for principal component analysis that simultaneously achieves dimensionality reduction and model interpretability in high-dimensional data scenarios.

\section{Acknowledgments}\label{sec7}
This work is supported by the National Natural Science Foundation of
China [Grant No. 12371281].The authors would like to express their sincere gratitude to the Associate Editor and the referees for their invaluable comments, which significantly contributed to enhancing the paper.

\end{document}